\documentclass{article}
\usepackage{ijcai24}

\usepackage{times}
\usepackage{soul}
\usepackage{url}
\usepackage[hidelinks]{hyperref}
\usepackage[utf8]{inputenc}
\usepackage[small]{caption}
\usepackage{graphicx}
\usepackage{amsmath}
\usepackage{amsthm}
\usepackage{booktabs}
\usepackage{algorithm}
\usepackage{algorithmic}
\usepackage[switch]{lineno}

\usepackage{color}
\usepackage[inline]{enumitem}
\usepackage{multirow}
\usepackage{bm}
\usepackage{amsfonts}
\usepackage{nicefrac}

\usepackage{longtable, ltxtable}
\usepackage{tablefootnote}

\newtheorem{theorem}{Theorem}

\newcommand{\E}[1]{\mathrm{E}\left(#1\right)}
\renewcommand{\H}[1]{\mathrm{H}\left(#1\right)}
\renewcommand{\P}{\mathrm{P}}
\newcommand{\kld}[2]{D_{\mathrm{KL}}\left(#1 \parallel #2 \right)}

\title{Building Universal Foundation Models for Medical Image Analysis with Spatially Adaptive Networks}

\author{
    Lingxiao Luo\and
    Xuanzhong Chen\and
    Bingda Tang\and
    Xinsheng Chen\and \\
    Rong Han\and
    Chengpeng Hu\and
    Yujiang Li\and
    Ting Chen
\affiliations
Tsinghua University
}

\begin{document}

\maketitle



\begin{abstract}
    Recent advancements in foundation models, typically trained with self-supervised learning on large-scale and diverse datasets, have shown great potential in medical image analysis.
    However, due to the significant spatial heterogeneity of medical imaging data, current models must tailor specific structures for different datasets, making it challenging to leverage the abundant unlabeled data.
    In this work, we propose a universal foundation model for medical image analysis that processes images with heterogeneous spatial properties using a unified structure.
    To accomplish this, we propose spatially adaptive networks (SPAD-Nets), a family of networks that dynamically adjust the structures to adapt to the spatial properties of input images, to build such a universal foundation model.
    We pre-train a spatial adaptive visual tokenizer (SPAD-VT) and then a spatial adaptive Vision Transformer (SPAD-ViT) via masked image modeling (MIM) on 55 public medical image datasets.
    The pre-training data comprises over 9 million image slices, representing the largest, most comprehensive, and most diverse dataset to our knowledge for pre-training universal foundation models for medical image analysis.
    The experimental results on downstream medical image classification and segmentation tasks demonstrate the superior performance and label efficiency of our model.
    Our code is available at \url{https://github.com/function2-llx/PUMIT}.
\end{abstract}

\section{Introduction}

\begin{figure}[htbp]
    \centering
    \includegraphics[width=.85\linewidth]{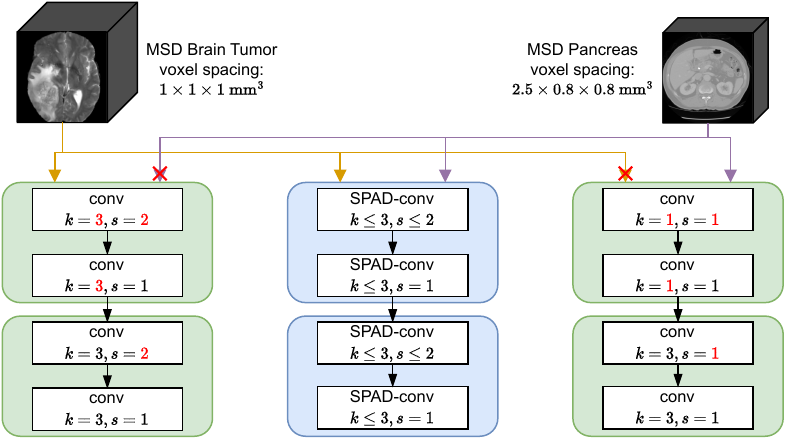}
    \caption{
        The first two stages of U-Net encoders designed by nnU-Net for Task 1 (left) and Task 5 (right) of MSD challenge~\protect\cite{MSD}, and the version of our proposed SPAD-Nets (middle).
        Convolution parameters (\(k\) for kernel size, \(s\) for stride) along the depth dimension are indicated.
        The structures designed by nnU-Net for datasets with different spacing have incompatible parts (marked as {\color{red} red}).
        The SPAD-Nets are able to handle images from both tasks by adapting structures to input spatial properties.
    }
    \label{fig:structures}
\end{figure}

Conventional deep learning models require a substantial amount of labeled data to be trained, but the availability of labeled data is often limited in domains such as health and medicine due to the expensive acquisition process.
Moreover, most of these learning models are developed in a task-specific manner, where individual models are designed and trained on a narrow range of data, further limiting the applications of models.
Recently, the field of artificial intelligence (AI) has witnessed a paradigm shift from training task-specific models to developing foundation models, which are typically trained on large-scale and diverse datasets using self-supervised learning (SSL), and can be applied to various downstream tasks via fine-tuning or prompting~\cite{bommasani2021opportunities}.
In general vision domain, for instance, without the need for expensive annotations, foundation models can be pre-trained on raw images with masked image modeling (MIM) and achieve promising results in downstream vision tasks such as classification, detection, and segmentation~\cite{mae_cvpr2022,eva_CVPR2023}.

However, general visual foundation models cannot be applied to medical images directly due to the significant differences between natural images and medical images.
While most natural images are 2D and of RGB modality, medical images are diverse in terms of imaging modalities such as X-ray, ultrasound, CT, and MRI, and dimensions as both 2D and 3D images are commonly used.
Although 2D and 3D convolutional neural networks (CNNs) have been used for 2D and 3D images, respectively, their model structures are inherently incompatible with each other due to different spatial dimensions.
Previous works attempted to break such ``dimensionality barrier" by switching to Vision Transformers (ViTs)~\cite{vit_iclr2021,unimiss}, but CNNs still play indispensable roles in ViTs as their patch embedding layers.
On the other hand, 3D medical images that are made up of stacked 2D slices can have a wide range of voxel spacing, that is, the physical size of a voxel, commonly represented by an in-plane spacing \(s_{\mathrm{plane}}\) and a slice thickness \(s_{\mathrm{slice}}\).
When the ratio between \(s_{\mathrm{slice}}\) and \(s_{\mathrm{plane}}\) (\(\nicefrac{s_{\mathrm{slice}}}{s_{\mathrm{plane}}}\)) is overly large, the image content may change substantially across slices.
Neural networks typically process images based on voxel grids, which are assumed to have isotropic voxel spacing, without considering the actual physical size of voxels.
As a result, most mainstream models adapt their structures to align with the voxel spacing statistics of individual datasets~\cite{nnunet}.
Yet, such practice can make model structures incompatible among datasets with different voxel spacing statistics, as illustrated in Figure~\ref{fig:structures}.
As a result, developing foundation models for medical image analysis is challenging because of the lack of unified model structures capable of handling medical imaging data with heterogeneous spatial properties.

In this work, we develop a universal foundation model for a broad range of medical images, including both 2D and 3D images with multiple imaging modalities and diverse voxel spacing.
To accomplish this, we introduce the concept of \textit{spatially adaptive networks} (SPAD-Nets), a family of networks that dynamically adjust their structures to adapt to the spatial properties of input images.
SPAD-Nets are implemented by replacing the building blocks of conventional networks, such as convolutions, with spatially adaptive ones.
Specifically, we propose \textit{spatially adaptive convolution} (SPAD-conv) to serve as the building blocks of SPAD-Nets.
A SPAD-conv shares the same set of parameters as a conventional 3D convolution, and its parameters are dynamically transformed to adapt to the input.
For instance, given an input image with a large \(\nicefrac{s_{\mathrm{slice}}}{s_{\mathrm{plane}}}\) value, a SPAD-conv will reduce its convolution kernel size and stride along the depth dimension accordingly.
With models built with SPAD-Nets, we pre-train a spatially adaptive visual tokenizer (SPAD-VT) and then a spatially adaptive ViT (SPAD-ViT) via masked image modeling (MIM) on large-scale public medical image datasets.
The contributions of our work are summarized as follows:
\begin{itemize}
    \item We propose SPAD-Nets, a family of networks used for building universal foundation models for medical image analysis. Based on SPAD-Nets, we manage to develop the SPAD-ViT, which is, to our knowledge, the first universal foundation model for medical image analysis. 
    \item We train a spatially adaptive visual tokenizer (SPAD-VT) followed by a SPAD-ViT on 55 public medical image datasets with over 9 million image slices. This is the largest, most comprehensive, and most diverse dataset to our knowledge for pre-training universal foundation models for medical image analysis.
    \item Our pre-trained models are fine-tuned on downstream medical image classification and segmentation tasks. The experimental results demonstrate the superior performance and label efficiency of our model.
\end{itemize}

\section{Related Works}



\paragraph{Universal Models for Medical Image Analysis}

A universal model capable of processing medical images with diverse spatial properties using a unified model structure is crucial for leveraging a large amount of unlabeled data available when developing foundation models.
One straightforward approach to building such a universal model is to use a 2D model to process multiple 2D slices split from a 3D volume \cite{joint_ssl_aaai2023}.
However, using a 2D model also requires additional mechanism to address the cross-plane contextual information inherent in 3D images.
UniMiSS \cite{unimiss} presents an alternative approach based on the ViT model, which converts the input 2D or 3D image into a sequence of 1D image patches using a switchable patch embedding (SPE) layer. The sequence is then processed by transformer blocks.
Additionally, the weight inflation technique \cite{pmlr-v193-zhang22a} is also proposed to adapt pre-trained weights from 2D models to 3D models, which inspires the development of SPAD-Nets.

\paragraph{Masked Image Modeling}

Inspired by the masked language modeling (MLM) in NLP \cite{devlin-etal-2019-bert}, masked image modeling (MIM) with transformer-based models has emerged as a viable SSL approach in the CV domain.
Some representative works include MAE \cite{mae_cvpr2022} with pixel value regression, BEiT family \cite{bao2022beit} with visual token reconstruction, and EVA series \cite{eva_CVPR2023,fang2023eva02} with vision feature prediction.
However, there have been fewer attempts to apply MIM to medical image analysis \cite{chen2023masked}, largely due to the significant differences between medical images and natural images.

\section{Preliminary}
\label{sec:pre}

In this section, we briefly review key concepts about spatial properties in medical images.
We then discuss how these concepts influence the design of current mainstream models for individual tasks of medical image analysis, and how they complicate the development of a foundation model.


A 3D medical image, such as a CT scan, is a stack of cross-sectional 2D image slices of a human body generated by a medical imaging device~\cite{Buzug2011}.
Similar to a typical digital image that can be represented by a 2D array of pixel values, each slice contains a 2D intensity array, and arrays of all slices are stacked to form the 3D array of voxel intensities for the whole 3D image.
The spatial dimensions within the slice plane are referred to as the \textit{in-plane} dimensions, and the dimension along which slices are stacked is referred to as the \textit{depth} dimension.
A 3D medical image contains information on \textit{voxel spacing}, which is the physical size of a voxel along each spatial dimension and is represented as \(s_{\mathrm{slice}} \times s_{\mathrm{h}} \times s_{\mathrm{w}}~(\mathrm{mm}^3)\).
Here, \(s_{\mathrm{slice}}\) is the slice thickness (or equivalently, the distance between adjacent slices), \(s_{\mathrm{h}}\) and \(s_{\mathrm{w}}\) are spacings along in-plane dimensions. 
In our work, we treat 2D images as a special case of 3D images with only one slice and a sufficiently large \(s_{\mathrm{slice}}\).
Moreover, all images in this work have an isotropic in-plane spacing, \(s_{\mathrm{h}} = s_{\mathrm{w}}\), which we denote as \(s_{\mathrm{plane}}\).

\begin{figure*}[!t]
    \centering
    \includegraphics[width=.9\linewidth]{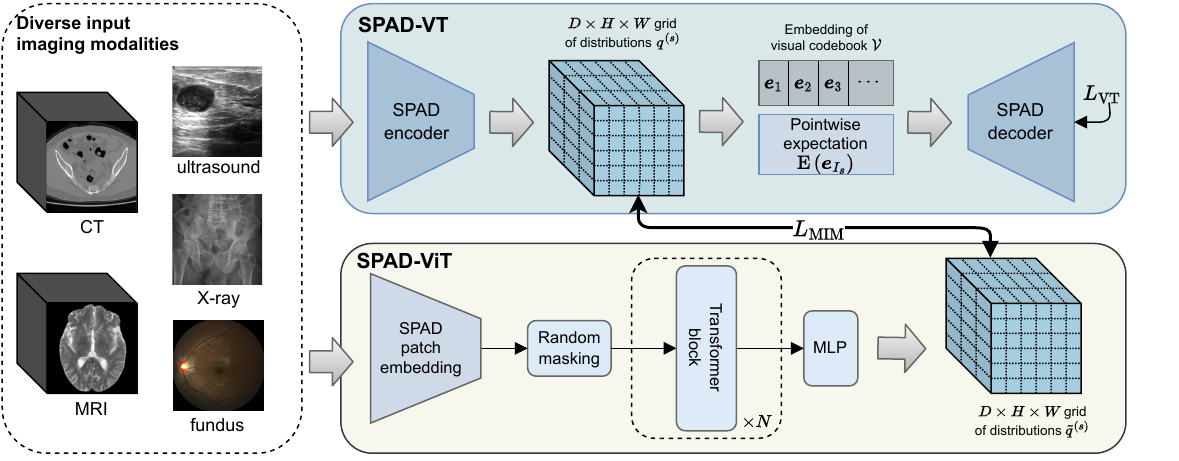}
    \caption{
        Illustration of our pre-training framework for both SPAD-VT and SPAD-ViT.
        The models built with our proposed SPAD-Nets can process a wide range of images using a unified model structure.
        Note that SPAD-VT is trained first (optimizing \(L_{\mathrm{VT}}\)), and is fixed during the training of SPAD-ViT (optimizing \(L_{\mathrm{MIM}}\)).
    }
    \label{fig:pumit}
\end{figure*}

The voxel spacing is crucial for accurate and comprehensive feature extraction from a medical image, such as measuring the physical volume of an object or the distance between objects.
However, general neural networks only operate on the voxel grids of the intensity array that are assumed to have isotropic voxel spacing but disregard the actual physical size of voxels.
To adapt a model to a dataset, nnU-Net~\cite{nnunet}, the widely used approach for medical image segmentation, proposes the following two empirical strategies:
\begin{enumerate*}[label=(\arabic*)]
    \item Resampling all images to the same target spacing, such as the median value of spacings of all training images.
    \item Adapting the network structure when the voxel spacing is significantly anisotropic, that is, the ratio between \(s_{\mathrm{slice}}\) and \(s_{\mathrm{plane}}\) is far from (mostly, much larger than) \(1\).
    Take the downsampling stage in the U-Net encoder illustrated in Figure~\ref{fig:structures} as an example: the feature map will not be downsampled along the depth dimension at a stage if the value of \(\nicefrac{s_{\mathrm{slice}}}{s_{\mathrm{plane}}}\) is larger than \(2\).
\end{enumerate*}

These strategies have been validated to be effective for individual tasks by nnU-Net with extensive experiments and largely adopted by later works in 3D medical image analysis~\cite{nnformer2023tip,Brugnara2023nc}.
However, both strategies become infeasible when developing a foundation model on large-scale and diverse datasets due to the following reasons:
\begin{enumerate*}[label=\arabic*)]
    \item Resampling all images from diverse datasets to the same target spacing will result in significant information loss and resampling artifacts, which can destroy delicate details in medical images.
    \item The network structures adapted for individual tasks are static and can be incompatible with other datasets, whereas a unified structure is essential to leverage the large-scale data together.
\end{enumerate*}
In conclusion, it is critical to develop new strategies to address the heterogeneous spatial properties of medical images that are applicable across diverse datasets for the development of foundation models for medical image analysis.

\section{Methodology}

\subsection{Overview}

Our solution to the challenge described in Section~\ref{sec:pre} is to design a universal model for a broad range of medical images with diverse spatial properties, built with our proposed SPAD-Nets.
At each stage of the model, the structure can be dynamically adjusted to the one that is adapted to the input.
We are then able to devise a pre-training framework based on masked image modeling (MIM) with visual token reconstruction~\cite{bao2022beit}, illustrated in Figure~\ref{fig:pumit}.
For an input image, a pre-trained visual tokenizer converts it to a sequence of tokens, where each token represents a patch on the image.
With random patches on the image masked out, the training objective is to predict tokens of masked patches given the visible patches.
As there is no pre-trained visual tokenizer for medical images available, especially a universal one, we also focus on the development of a visual tokenizer first.

\subsection{SPAD-Nets}
\label{sec:method-spad-nets}

SPAD-Nets are implemented by replacing the building blocks of conventional networks with spatially adaptive ones.
Specifically, we propose the \textit{spatially adaptive convolution} (SPAD-conv) to serve as the convolutional building blocks of SPAD-Nets.
In general, a SPAD-conv shares the same set of parameters as a conventional 3D convolution, which is termed as its \textit{base network} and has isotropic convolution kernel sizes and strides across all spatial dimensions.
When encountering an input with anisotropic spacing, the parameters are transformed accordingly to adapt to the input before applying to it.
The design idea of SPAD-conv is largely inspired by and modified from the empirical rules of designing a static network structure adapting to images of a dataset proposed by nnU-Net.
Given an input feature map, let \(s_{\mathrm{plane}}\) and \(s_{\mathrm{slice}}\) be the in-plane spacing and slice thickness as introduced in Section~\ref{sec:pre}, respectively.
The general principle for designing SPAD-conv can be summarized as follows:
\begin{enumerate*}[label=(\arabic*)]
    \item if \(s_{\mathrm{slice}} < 2s_{\mathrm{plane}}\), the anisotropy of the input is considered to be tolerable and the input can be directly processed by the base network;
    \item otherwise, when \(s_{\mathrm{slice}} \ge 2s_{\mathrm{plane}}\), the input is considered to be overly anisotropic where changes across slices can be significant, and the base network must adjust the structure to reduce the information aggregation across the depth dimension.
\end{enumerate*}
Based on this principle, we design three types of SPAD-conv: downsampling, k3s1, and upsampling, whose base networks are used as fundamental components in common vision models, such as U-Net.
We now detail these types of SPAD-conv with an example of turning a conventional U-Net into a SPAD-Nets counterpart.




\paragraph{Downsampling}

In the U-Net encoder, the feature map is typically downsampled with a ratio of 2 at each stage through a convolution with a stride of 2 and a kernel size of 2 or 3.
When using the downsampling type of SPAD-conv with the same convolution as base network, given an overly anisotropic input (\(s_{\mathrm{slice}} \ge 2s_{\mathrm{plane}}\)), the downsampling along the depth dimension will be disabled.
Specifically, the weight of convolution kernel along the depth dimension is reduced to the size of \(1\) through sum pooling, and the stride along the depth dimension is similarly adjusted to \(1\).
The voxel spacing of the output feature map is multiplied by the adjusted stride.

\paragraph{K3s1}

The convolution with a kernel size of \(3\) and a stride of 1, abbreviated as k3s1 by us, is used in both the U-Net encoder and decoder for feature extraction.
When using the k3s1 type of SPAD-conv with the same convolution as base network, given an overly anisotropic input (\(s_{\mathrm{slice}} \ge 2s_{\mathrm{plane}}\)), the feature aggregation across the depth dimension is disabled by reducing the weight of convolution kernel along the depth dimension to the size of 1 through sum pooling.

\paragraph{Upsampling}

In the U-Net decoder, the feature map is typically upsampled with a ratio of 2 at each stage through a transposed convolution with a stride of 2 and a kernel size of 2 or 3.
When using the upsampling type of SPAD-conv with the same transposed convolution as base network, however, the feature map is upsampled along the depth dimension if and only if it is downsampled along the depth dimension within the corresponding symmetric stage in the encoder.
This is because the upsampling is considered to be a reverse process of downsampling in terms of spatial properties.

We specifically employ sum pooling for weight reduction to maintain the mean and variance of output feature values, similar to the weight inflation technique used to adapt pre-trained weights from 2D models to 3D models~\cite{pmlr-v193-zhang22a}.
While we have demonstrated the integration of SPAD-conv into U-Net for clarity, we note that SPAD-conv can be similarly applied to a wide range of convolution-based networks and convert them to SPAD-Nets.

\paragraph{Generalized Resampling}
Beyond the aforementioned usages in U-Net, both downsampling and upsampling types of SPAD-conv can be generalized with a base convolution network with both kernel size and stride being \(2^k\), where \(k\) is a positive integer.
Such network resamples the feature map with a ratio of \(2^k\), and it can be used as the patch embedding in ViT or the output layer for some dense prediction tasks, such as segmentation and reconstruction.
We define the \textit{degree of anisotropy} (DA) of a feature map to be the minimum number of normal downsampling with a ratio of 2 performing on it to have \(s_{\mathrm{slice}} < 2s_{\mathrm{plane}}\) after downsampling, which yields \(\mathrm{DA}=\max\left \{0, \left \lfloor \log_2 \nicefrac{s_{\mathrm{slice}}}{s_{\mathrm{plane}}} \right \rfloor\right\}\).
Since such a network can be viewed as \(k\) consecutive normal resampling, for generalized downsampling, the effective number of downsampling along the depth dimension is therefore reduced to \(k_0=\max\{k - \mathrm{DA}, 0\}\), with corresponding kernel size and stride reduced to \(2^{k_0}\).
The generalized upsampling can be achieved similarly by reversing the process.

\subsection{Visual Tokenizer Development}
\label{sec:method-pt}

We first briefly review the concept of the visual tokenizer.
The visual tokenizer is implemented using a VQ-VAE \cite{vqvae}, a variant of VAE \cite{vae} where the latent variables follow discrete categorical distributions over the visual codebook $\mathcal{V}$, a finite set of tokens.
Specifically, an encoder learns a posterior distribution $q(z \mid x)$ to output a \(D \times H \times W\) grid of categorical token distributions over \(\mathcal{V}\) given an image \(x\).
Token samples are drawn from these categorical distributions, and converted to a $d$-dimensional embedding vector from an 
embedding space \(\bm{e} \in \mathbb{R}^{|\mathcal{V}| \times d}\), known as the vector quantization (VQ) process.
A decoder learns to reconstruct the original image given the embeddings of the drawn token samples.

We build a spatially adaptive visual tokenizer (SPAD-VT) as follows.
The encoder resembles a typical convolutional backbone, but is equipped with k3s1 and downsampling types of SPAD-conv.
Given an input image with voxel spacing information, it outputs a feature map with \(\nicefrac{1}{16}\) of the original resolution on the in-plane dimensions and \(
    \min\left\{\nicefrac{2^{\mathrm{DA}}}{16}, 1\right\}
\) of the original resolution on the depth dimension.
The decoder is composed of a stack of k3s1 and upsampling types of SPAD-conv.
It gradually upsamples the quantized feature map output by the encoder, and finally, it outputs a reconstruction of the original image. To make the VQ process differentiable, token samples are typically drawn from the categorical distributions using Gumbel-Softmax reparameterization~\cite{maddison2017the,jang2017categorical} and straight-through estimators~\cite{bengio2013estimating}.
However, these techniques can suffer from issues of severe codebook collapse, low codebook utilization, and perturbed reconstruction process introduced by stochastic quantization~\cite{Zhang_2023_CVPR}.
We propose several enhancements to make SPAD-VT produce more robust tokenization output.

\paragraph{Soft Token Representation}

Instead of sampling a discrete token from the categorical distribution, we directly represent a \textit{soft token} with the distribution.
The embedding of a soft token is calculated as the expectation of the token embedding.
Specifically, let random variable $I$ be the index of a token in the codebook.
The soft token embedding is given as:
\begin{equation}
    \E{\bm{e}_{I}} = \sum_{i=1}^{|\mathcal{V}|} \P(I=i)\bm{e}_i.
\end{equation}
This strategy could mitigate the codebook collapse issue, while still preserving a deterministic reconstruction process.

\paragraph{Dual Prior Distribution Regularization}

To mitigate the low codebook utilization issue, we adopt the prior distribution regularization (PDR) technique~\cite{Zhang_2023_CVPR} and generalize it to the soft token representation.
Let \(q^{(s)}\) be the distribution of the token with spatial position \(s\) output by the encoder.
PDR assumes the prior that the codebook should be uniformly utilized.
Thus, the averaged distribution of \(q^{(s)}\) of all spatial positions should be close to the uniform distribution $p_{\mathrm{prior}}$.
The objective of PDR can be represented as minimizing $\kld{\E{q^{(S)}}}{p_{\mathrm{prior}}}$, where \(S\) is a uniformly distributed random variable of the token's spatial position.
An interpretation of PDR is available in the appendix.
On the other hand, to prevent all distributions collapsing to the uniform distribution, we introduce another dual objective maximizing $\E{\kld{q^{(S)}}{p_{\mathrm{prior}}}}$ to increase the sharpness of the learned distributions.
Since $p_{\mathrm{prior}}$ is a uniform distribution, it can be shown that for any distribution $p$, $\kld{p}{p_{\mathrm{prior}}}=-\H{p} + \ln \left \lvert \mathcal{V} \right \rvert$, where $\H{p} = -\sum_{i=1}^{|\mathcal{V}|} p_i \ln p_i$ (entropy of \(p\)).
Therefore, the dual PDR is implemented to minimize the following objective:
\begin{equation}
    L_{\mathrm{reg}} = -\lambda_1 \H{\E{q^{(S)}}} + \lambda_2\E{\H{q^{(S)}}},
\end{equation}
where $\lambda_1, \lambda_2$ are positive hyper-parameters.
The final optimization objective for SAPD-VT is a weighted combination of the objective of VQGAN~\cite{VQGAN} and \(L_{\mathrm{reg}}\).

\subsection{Universal Foundation Model}

With previous preparation, we are now able to realize a universal foundation model for medical image analysis by building a spatially adaptive ViT (SPAD-ViT) and using the SPAD-VT to perform MIM with visual token reconstruction.

\paragraph{Architecture}

We employ a plain ViT macro architecture that is composed of a patch embedding layer and a stack of transformer blocks~\cite{vit_iclr2021}.
The patch embedding layer divides the input image into a sequence of non-overlapping patches and linearly projects each patch to an embedding vector.
This process is essentially equivalent to processing the input image with a downsampling convolution that has identical kernel size and stride, and can be replaced by a downsampling type of SPAD-conv.
The micro architecture of the transformer block follows EVA-02~\cite{fang2023eva02}, which incorporates several advancements in language models including the 2D rotary position embedding (RoPE) \cite{RoPE}.
We further extend 2D RoPE to work with 3D medical images, as detailed in the appendix.

\paragraph{MIM Pre-training}

We pre-train the SPAD-ViT via visual token reconstruction with our pre-trained SPAD-VT, which is similar to BEiT~\cite{bao2022beit}, but the visual tokens are in soft representations.
For a given 3D image $x$, it is converted into a \(D \times H \times W\) grid of token distributions \(\left(q^{(s)}\right)_{s \in G}\) by the pre-trained SPAD-VT, where \(G\) is the set of all spatial positions on the grid.
On the other hand, the patch embedding layer of ViT also divides $x$ into corresponding $D \times H \times W$ patches.
We randomly mask out 55\% of the patches and replace their embeddings with a learnable embedding \(\bm{e}_{\mathrm{mask}}\).
Transformer blocks then process the corrupted patch embeddings and output a \(D \times H \times W\) grid of hidden vectors.
These hidden vectors are projected to the categorical distributions \(\left(\tilde{q}^{(s)}\right)_{s \in G}\) over $\mathcal{V}$ through a multi-layer perceptron (MLP) with a single hidden layer.
The pre-training objective is minimizing the cross entropy of the reconstructed distribution relative to the distribution output by the SPAD-VT:
\begin{equation}
L_{\mathrm{MIM}} = \frac{1}{\left\lvert G_{\mathrm{M}} \right\rvert} \sum_{s \in G_{\mathrm{M}}} \sum_{i=1}^{\left\lvert \mathcal{V} \right\rvert} q^{(s)}_i\ln \tilde{q}^{(s)}_i,
\end{equation}
where $G_{\mathrm{M}} \subset G$ is the set of positions of all masked patches.
We intentionally choose a higher masking ratio for 3D medical images compared to previous works on 2D images.
Since patches from multiple slices may provide additional information for reconstruction, a higher masking ratio can help balance the difficulty of the pre-training task.

\section{Experiments}

\subsection{Pre-training}

\paragraph{Dataset}

We collect 55 public medical image datasets for both SPAD-VT and SPAD-ViT pre-training, including 15 datasets of 2D images and 40 datasets of 3D images.
The collected data include a wide range of image modalities (CT, MRI, X-ray, ultrasound and fundus photography) and body parts (abdomen, brain, breast, cervix, chest, fetal, fundus, head-neck, heart, pelvis, prostate, and spine), with diverse spatial properties.
We compare the size and imaging modalities of data for pre-training with related works in Table~\ref{table:data-cmp}.
We pre-process the data to a unified format, including moving the depth dimension of an image to the first spatial dimension.
During pre-training, we only use raw images in a self-supervised manner without using any annotations or labels from the datasets.
For each modality in each dataset, we randomly exclude one image from training data for validation.
The details of all datasets used for pre-training and the steps of pre-processing are provided in the appendix.

\begin{table}[htbp]
    \centering
    \begin{tabular}{c|ccc}
        \toprule
        \textbf{Work} & \textbf{\#3D} & \textbf{\#2D} & \textbf{Modality} \\ 
        \midrule
        Swin UNETR & 5,050 & 0 & CT \\
        SMIT & 3,643 & 0 & CT \\
        \citeauthor{joint_ssl_aaai2023} & 3,694 & 108,948 & CT, XR \\
        UniMiSS & 5,022 & 108,948 & CT, XR \\
        \midrule
        Ours & 48,344 & 348,991 & \begin{tabular}[c]{@{}c@{}}CT, MRI, \\ XR, US, FP\end{tabular} \\
        \bottomrule
    \end{tabular}
    \caption{
        The size and imaging modalities of data used for pre-training in related works and this work.
        Note: \#3D and \#2D: the number of 3D and 2D images, respectively; XR: X-ray; US: ultrasound; FP: fundus photography.
    }
    \label{table:data-cmp}
\end{table}

\begin{table*}[!t]
    \centering
    \fontsize{9}{11}\selectfont
    \begin{tabular}{c|cc|cc|cc|cc|cc|cc|cc}
    \toprule
    \multirow{2}{*}{\textbf{Method}} & \multicolumn{2}{c|}{\textbf{Organ}} & \multicolumn{2}{c|}{\textbf{Nodule}} & \multicolumn{2}{c|}{\textbf{Fracture}} & \multicolumn{2}{c|}{\textbf{Adrenal}} & \multicolumn{2}{c|}{\textbf{Vessel}} & \multicolumn{2}{c|}{\textbf{Synapse}} & \multicolumn{2}{c}{\textbf{Average}} \\
 & AUC & ACC & AUC & ACC & AUC & ACC & AUC & ACC & AUC & ACC & AUC & ACC & AUC & ACC \\
\midrule
ResNet-50+3D & 99.4 & 88.3 & 87.5 & 84.7 & 72.5 & 49.4 & 82.8 & 74.5 & 90.7 & 91.8 & 85.1 & 79.5 & 86.3 & 78.0 \\
ResNet-50+ACS & 99.4 & 88.9 & 88.6 & 84.1 & 75.0 & 51.7 & 82.8 & 75.8 & 91.2 & 85.8 & 71.9 & 70.9 & 84.8 & 76.2 \\
auto-sklearn & 97.7 & 81.4 & 91.4 & 87.4 & 62.8 & 45.3 & 82.8 & 80.2 & 91.0 & 91.5 & 63.1 & 73.0 & 81.5 & 76.5 \\
UniMiSS & \textbf{99.8} & 94.9 & 89.9 & 84.2 & 70.4 & 50.8 & 87.0 & 81.2 & 93.8 & 94.0 & 85.8 & 81.8 & 87.8 & 81.2 \\
SMIT & 99.5 & 91.1 & 88.4 & 85.2 & 63.1 & 49.6 & 84.8 & \textbf{82.2} & 89.7 & 91.1 & 70.8 & 73.3 & 82.7 & 78.8 \\
EVA-02-B & 99.2 & 88.0 & 91.0 & 85.8 & 75.2 & \textbf{56.7} & 86.7 & 79.5 & 95.8 & \textbf{95.6} & 72.1 & 75.3 & 86.7 & 80.2 \\
Ours & \textbf{99.8} & \textbf{95.6} & \textbf{94.3} & \textbf{88.4} & \textbf{77.6} & 55.4 & \textbf{87.8} & 81.2 & \textbf{97.3} & \textbf{95.6} & \textbf{96.4} & \textbf{91.5} & \textbf{92.2} & \textbf{84.6}
     \\
    \bottomrule
    \end{tabular}
    \caption{
        Classification results on the test sets of the 3D image datasets from MedMNIST v2.
        The results of ResNet-50 and auto-sklearn are taken from the original paper~\protect\cite{medmnistv2}.
        Both AUC and ACC values are shown in percentages.
    }
    \label{table:cls_3d}
\end{table*}

\paragraph{Implementation Details}

Our models are implemented based on MONAI \cite{cardoso2022monai}.
For SPAD-VT, we set codebook size $|\mathcal{V}|=1024$.
During pre-training, we randomly crop a $\max\{64 / 2^{\mathrm{DA}}, 1\} \times 128 \times 128$ patch from an image given its DA.
We adopt a batch sampling strategy in which all images within a batch are selected to have the same DA, ensuring that they have the same cropped size and thus can be processed by the network in parallel.
We train SPAD-VT for 600K steps with the batch size of 48 and the peak learning rate of $10^{-4}$.
For SPAD-ViT pre-training, we randomly crop a $\max\{80 / 2^{\mathrm{DA}}, 1\} \times 160 \times 160$ patch from an image, and use the same batch sampling strategy as the SPAD-VT training.
We initialize SPAD-ViT with the EVA-02-B pre-trained weights and train it for 300K steps with a batch size of 84 and a peak learning rate of $2 \times 10^{-5}$.
All pre-trained models are trained with AdamW optimizer \cite{loshchilov2018decoupled}.

\begin{figure}[!t]
    \centering
    \includegraphics[width=\linewidth]{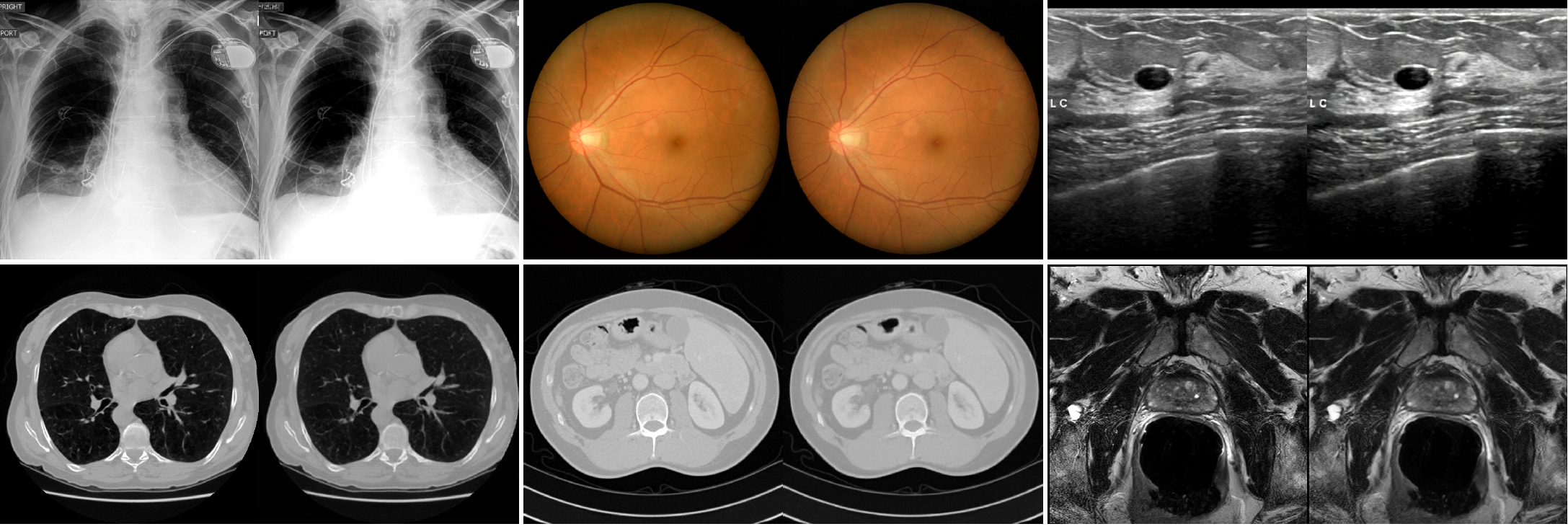}
    \caption{
        Reconstruction results of SPAD-VT on unseen images.
        The images are taken from different body parts with diverse imaging modalities and spatial properties.
        For each pair of images, the original image is on the left and the reconstructed image is on the right.
        The first row:
            (left) chest X-ray image from  CheXpert~\protect\cite{chexpert};
            (middle) fundus photograph from GAMMA;
            (right) breast ultrasound image from BUSI~\protect\cite{BUSI}.
        The second row:
            (left): lung CT image from LIDC-IDRI~\protect\cite{LIDC-IDRI};
            (middle): abdomen CT image from FLARE 2022;
            (right): prostate MRI image from PROSTATE-MRI~\protect\cite{prostate-mri}.
        More examples are available in the appendix.
    }
    \label{fig:tokenizer-rec}
\end{figure}

\begin{figure}[htbp]
    \centering
    \includegraphics[width=.9\linewidth]{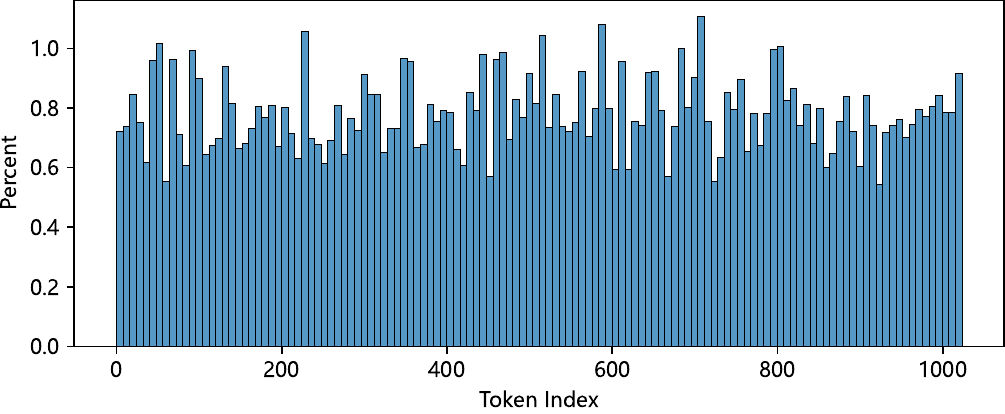}
    \caption{
        Token utilization histogram.
        Every eight tokens with consecutive indexes are merged into one rectangle for better visualization within the limited page width.
    }
    \label{fig:tokenizer-util}
\end{figure}

\paragraph{Results}

We demonstrate the reconstruction results of SPAD-VT on unseen validation images in Figure~\ref{fig:tokenizer-rec}.
SPAD-VT could process a wide range of medical images with heterogeneous imaging modalities and spatial properties using a unified model structure.
As shown in Figure~\ref{fig:tokenizer-rec}, images from diverse sources can be consistently reconstructed from the learned soft token representations with rich details preserved.
Moreover, we conduct the following analyses to evaluate the diversity and sharpness of the learned distributions.
First, we plot the codebook utilization on validation images in Figure~\ref{fig:tokenizer-util}.
The usage percentages of all tokens are in the same order of magnitude, indicating a balanced utilization of the codebook.
Second, we calculate the average value of $\E{\kld{q^{(S)}}{p_{\mathrm{prior}}}}$ on validation images as 2.791, while $\H{p_{\mathrm{prior}}}=\ln |\mathcal{V}|\approx6.931$.
This result indicates that the learned distributions preserve sharpness to some extent.

\subsection{Classification Experiments}

\paragraph{Dataset}

MedMNIST v2 \cite{medmnistv2} is a large-scale benchmark for medical image classification on standardized MNIST-like 2D and 3D images with diverse modalities, dataset scales, and tasks.
The size of each image is either $28 \times 28$ (2D) or $28 \times 28 \times 28$ (3D).
We primarily experiment on the 3D portion of MedMNIST v2, namely the Organ \cite{organmnist1,organmnist2}, Nodule \cite{LIDC-IDRI}, Fracture \cite{fracturemnist3d}, Adrenal, Vessel \cite{vesselmnist3d}, and Synapse datasets.
We also experiment on 3 tasks from the 2D portion, namely the Derma~\cite{dermamnist1,dermamnist2}, Retina~\cite{retinamnist}, and Tissue~\cite{tissuemnist} datasets.

\paragraph{Implementation Details}

Like common ViT-based models, our model performs classification tasks by feeding the hidden state of the \texttt{[CLS]} token into a fully-connected layer.
We conduct classification experiments by fine-tuning from our and several other self-supervised pre-trained models on each dataset.
Following the official split of training/validation/test sets, we train all models on the training sets for 100 epochs.
The model weights from the epoch with the highest AUC score on the validation set are selected for comparison on the test set.
The experimental details are given in the appendix.


\paragraph{Results}

Table~\ref{table:cls_3d} compares the classification performance of our model with other models on the 3D portion using area under the ROC curve (AUC)  and accuracy (ACC) metrics.
Our model consistently outperforms all other models across all datasets in terms of AUC, and achieves the highest ACC on 4 out of 6 datasets.
On average, our model improves AUC by at least 4.4\% and accuracy by at least 3.4\% across all datasets.
Notably, on the Synapse dataset, our model improves the performance by a significantly large margin with 10.6\% of AUC and 9.7\% of ACC.
The imaging modality of the Synapse dataset, electron microscope, is not covered by our pre-training data, showing the generalizability of our model.
Table~\ref{table:cls_2d} shows the classification performance on 2D tasks.
Our model outperforms all other baselines except for EVA-02 with a large margin, and is comparable to EVA-02.
We hypothesize that the large-scale pre-training in general vision domain and the natural language supervision of EVA-02 can benefit 2D medical image classification;
however, the results are still inferior on 3D images due to significant distribution shift.

\begin{table}[htbp]
    \centering
    \fontsize{9}{11}\selectfont
    \setlength{\tabcolsep}{2pt}
    \begin{tabular}{c|cc|cc|cc|cc}
    \toprule
    \multirow{2}{*}{\textbf{Method}} & \multicolumn{2}{c|}{\textbf{Derma}} & \multicolumn{2}{c|}{\textbf{Retina}} & \multicolumn{2}{c|}{\textbf{Tissue}} & \multicolumn{2}{c}{\textbf{Average}} \\
 & AUC & ACC & AUC & ACC & AUC & ACC & AUC & ACC \\
\midrule
ResNet-50 & 91.2 & 73.1 & 71.6 & 51.1 & 93.2 & 68.0 & 85.3 & 64.1\\
AutoML Vision & 91.4 & 76.8 & 75.0 & 53.1 & 92.4 & 67.3 & 86.3 & 65.7\\
UniMiSS & 90.9 & 74.1 & 71.2 & 52.5 & 92.8 & 67.6 & 85.0 & 64.7\\
EVA-02-B & 94.5 & \textbf{80.1} & \textbf{77.8} & 54.5 & \textbf{95.1} & \textbf{73.5} & \textbf{89.1} & \textbf{69.4}\\
Ours & \textbf{94.6} & 79.2 & 77.4 & \textbf{55.3} & 94.9 & 73.0 & 89.0 & 69.2\\
    \bottomrule
    \end{tabular}
    \caption{
        Classification results on the test sets of the 2D image datasets from MedMNIST v2.
        The results of ResNet-50 and Google AutoML Vision are taken from the original paper~\protect\cite{medmnistv2}.
        Both AUC and ACC values are shown in percentages.
    }
    \label{table:cls_2d}
\end{table}

\subsection{Segmentation Experiments}

\paragraph{Dataset}

We use the following two datasets to evaluate the segmentation performance.
1) The Beyond the Cranial Vault (BTCV) dataset, which contains 30 CT scans with 13 classes of regions annotated. 
We follow the train/validation split of SMIT \cite{smit} to use 21 scans for training and 9 scans for validation.
2) The MRI subset of Combined Healthy Abdominal Organ Segmentation (CHAOS) dataset, which contains 40 pairs of T1-DUAL MRI images and 40 T2-SPIR MRI images with 4 classes of regions annotated.
The segmentation labels of 20 pairs of T1-DUAL images and 20 T2-SPIR images are available for training, while the labels for remaining cases are hidden for online testing.

\paragraph{Implementation Details}

For both tasks, we resample the images to the median spacing among the training data. 
Our pre-trained SPAD-ViT is used to initialize the backbone, which extracts feature maps with the fixed in-plane scale of \(\frac{1}{16}\).
To build hierarchical feature maps for the segmentation head, following the implementation of EVA-02 \cite{fang2023eva02}, we process the extracted feature maps with transposed convolutions to obtain feature maps with the in-plane scales of $\frac{1}{8}$ and $\frac{1}{4}$, and a max-pooling layer to obtain a feature map with the in-plane scale of $\frac{1}{32}$.
Subsequently, a plain U-Net decoder processes the hierarchical feature maps to obtain the segmentation output.
All segmentation models for both tasks are trained for 50K steps and the final weights are used for evaluation.
Dice coefficient, average symmetric surface distance (ASSD), and Hausdorff distance (HD) are used to evaluate results for both tasks.
The details on experimental settings and evaluation metrics are given in the appendix.

\paragraph{Results}

\begin{table}[htbp]
    \centering
    \fontsize{9}{11}\selectfont
    \setlength{\tabcolsep}{4.1pt}
    \begin{tabular}{c|ccc|ccc}
    \toprule
    \multirow{2}{*}{\textbf{Method}} & \multicolumn{3}{c|}{\textbf{BTCV} (validation)}          & \multicolumn{3}{c}{\textbf{CHAOS} (hidden test)} \\
               & Dice  & ASSD \(\downarrow\) & HD \(\downarrow\) & Dice  & ASSD & HD    \\
    \midrule
    TFS        & 84.86 & 1.59 & 18.90 & 89.96 & 2.18 & 25.39 \\
    Swin UNETR & 81.06 & 3.20 & 92.33 & 85.51 & 6.66 & 68.02 \\
    UniMiSS    & 85.17 & 1.01 & 17.78 & 90.49 & 1.91 & 24.94 \\
    EVA-02-B   & 86.51 & 0.86 & 16.91 & 91.48 & 1.71 & 20.69 \\
    Ours                             & \textbf{87.31} & \textbf{0.74} & \textbf{13.12} & \textbf{92.09} & \textbf{1.54} & \textbf{19.85} \\
    \bottomrule
    \end{tabular}
    \caption{
        Segmentation results on the validation set of BTCV and hidden test set of CHAOS.
        ``TFS" indicates training a plane ViT backbone from scratch.
        The Dice values are shown in percentages.
    }
    \label{table:seg}
\end{table}

Table~\ref{table:seg} compares the segmentation performance of our model with other models on the BTCV validation set and CHAOS hidden test set.
The models for comparison include Swin UNETR, the previous SOTA on MSD~\cite{MSD}.
Our model consistently outperforms all other models in terms of all 3 metrics for both tasks.
On the BTCV task, our model improves the HD metric by a large margin of \(3.19\).
On the CHAOS tasks, our model could still achieve evident improvement although the baseline performance is almost saturate.

\begin{table}[htbp]
    \centering
    \fontsize{9}{11}\selectfont
    \setlength{\tabcolsep}{3pt}
    \begin{tabular}{c|ccc|ccc|ccc}
    \toprule
    \multirow{2}{*}{$r$} & \multicolumn{3}{c|}{Dice} & \multicolumn{3}{c|}{ASSD} & \multicolumn{3}{c}{HD} \\
     & TFS & Ours & $\Delta$ & TFS & Ours & $\Delta$ & TFS & Ours & $\Delta$ \\
    \midrule
    0.2 & 70.98 & 77.84 & 6.85 & 4.96 & 2.39 & 2.57 & 51.13 & 25.56 & 25.57 \\
    0.4 & 75.11 & 81.28 & 6.17 & 4.67 & 2.41 & 2.26 & 43.85 & 21.71 & 22.14 \\
    0.6 & 80.73 & 84.87 & 4.14 & 2.74 & 1.23 & 1.51 & 31.70 & 18.03 & 13.67 \\
    0.8 & 84.73 & 86.96 & 2.23 & 1.61 & 0.78 & 0.83 & 23.78 & 13.66 & 10.12 \\
    1.0 & 84.86 & 87.31 & 2.45 & 1.59 & 0.74 & 0.86 & 18.90 & 13.12 & 5.78  \\
    \bottomrule
    \end{tabular}
    \caption{
        Segmentation results on the BTCV validation set comparing models trained from scratch (TFS) with those fine-tuned from our pre-trained weights, using different ratios $r$ of training data.
        The columns of ``$\Delta$" indicate the amount of improvement.
    }
    \label{table:seg-eff}
\end{table}

\paragraph{Few-shot Learning}

To investigate how our model improves the label efficiency of the downstream tasks, we compare the model trained from scratch and fine-tuned from our model with different ratios of training data in Table~\ref{table:seg-eff}.
Fine-tuning from our pre-trained model consistently outperforms training from scratch with a large margin across all metrics, and the improvements are more significant when fewer data are available.
Moreover, our pre-trained model could effectively mitigate the requirement for labeled data.
As shown in Table~\ref{table:seg-eff}, even when fine-tuning with 60\% of the training data, our pre-trained model still achieves better results than a model trained from scratch using the full dataset.
Notably, our model still preserves decent performance with only 4 labeled samples (20\% of the training data).

\section{Conclusion}

In this work, we develop the first universal foundation model for medical image analysis, to our knowledge.
To address the heterogeneous spatial properties of medical images, we introduce the concept of SPAD-Nets, a family of networks which dynamically adjust the structures to adapt to the spatial properties of input images.
We build SPAD-ViT which is able to leverage 55 diverse datasets for pre-training, representing the most extensive pre-training effort for universal foundation models for medical image analysis to our knowledge.
In addition, we develop SPAD-VT, a universal visual tokenizer for a broad range of medical images, offering a robust reconstruction objective for MIM, and may also be used in multimodal models in future works.
The pre-trained SPAD-ViT is demonstrated with superior downstream performance and label efficiency.
We hope our work could inspire future works on developing foundation models for medical image analysis.

\appendix


\bibliographystyle{named}
\bibliography{ijcai24}

\clearpage


\section{Details on RoPE}

\subsection{From 1D to 2D RoPE}
Positional embeddings play a crucial role in transformer models \cite{attention-is-all-you-need}.
RoPE \cite{RoPE} is a kind of position embedding based on rotation transformation.
We extend the 2D rotary position embedding (RoPE) \cite{RoPE} to work with 3D medical images.

For a $d$-dimensional query or key vector $\bm{a}=(a_0, a_1, \dots, a_{d-1})$ in an attention block, the vanilla RoPE, primarily designed for sequence data, rotates each pair \((a_{2i}, a_{2i+1})\) (where \(i \in \{0, 1, \cdots, \frac{d}{2} - 1\}\)) on the 2D plane by an angle of \(\omega_it\) (in radian), where \(t\) is the position index and it is commonly chosen that \(\omega_i=b^{-\frac{2i}{d}}\) with a base $b$.
This can be represented by matrix multiplication $\mathcal{R}(t) \bm{a}$, the transformation matrix is defined as
\begin{equation}
    \mathcal{R}(t)=\mathrm{diag}\left(R(\omega_0 t), \dots, R(\omega_{\frac{d}{2}-1}t)\right),
\end{equation}
where $R(\theta)$ is defined as the 2D rotation matrix that rotates an angle of $\theta$:
\begin{equation}
    R(\theta) = 
    \begin{bmatrix}
        \cos \theta & -\sin \theta \\
        \sin \theta & \cos \theta
    \end{bmatrix}.
\end{equation}

When handling 2D image data, for a given position on a 2D image coordinated with $(t_x, t_y)$, 2D RoPE extends 1D RoPE by rotating the first half of the vector $\bm{a}$ according to $t_x$ and the second half according to $t_y$.
The corresponding transformation matrix \(\mathcal{R}(t_x, t_y)\) is given by:
\begin{equation}
    \mathcal{R}(t_x, t_y) =     
    \begin{bmatrix} 
        \mathcal{R}(t_x) & \\
         & \mathcal{R}(t_y) 
    \end{bmatrix}.
\end{equation}


Both vanilla 1D RoPE and extended 2D RoPE preserves the following properties. Take 2D RoPE for example:
\begin{enumerate}
    \item \textbf{Relativity}: for two positions $(t_x^{(1)}, t_y^{(1)})$ and $(t_x^{(2)}, t_y^{(2)})$, it can be easily verified that $\mathcal{R}^\mathsf{T}(t_x^{(1)}, t_y^{(1)}) \mathcal{R}(t_x^{(2)}, t_y^{(2)}) = \mathcal{R}(t_x^{(2)} - t_x^{(1)}, t_y^{(2)} - t_y^{(1)})$.
        Consequently, for two vectors $\bm{a}$ and $\bm{b}$, the result of their inner product after applying RoPE is a quantity dependent on the relative position $(t_x^{(2)} - t_x^{(1)}, t_y^{(2)} - t_y^{(1)})$:
        \begin{align}
         &\left(\mathcal{R}(t_x^{(1)}, t_y^{(1)})\bm{a}\right)^\mathsf{T} \mathcal{R}(t_x^{(2)}, t_y^{(2)})\bm{b} \\
        = & \bm{a}^\mathsf{T} \mathcal{R}^\mathsf{T}(t_x^{(1)}, t_y^{(1)}) \mathcal{R}(t_x^{(2)}, t_y^{(2)}) \bm{b} \\
        = & \bm{a}^\mathsf{T} \mathcal{R}(t_x^{(2)} - t_x^{(1)}, t_y^{(2)} - t_y^{(1)}) \bm{b}.
        \end{align}
    \item \textbf{Reversibility}: $\mathcal{R}(t_x^{(1)}, t_y^{(1)})=\mathcal{R}(t_x^{(2)}, t_y^{(2)})$ if and only if \((t_x^{(1)}, t_y^{(1)}) = (t_x^{(2)}, t_y^{(2)})\), as long as the value of each $\omega_i$ are appropriately chosen.
\end{enumerate}

\subsection{Additive 3D RoPE}

We aim to extend RoPE to 3D images to handle the 3D coordinate \((t_x, t_y, t_z)\), while maintaining the aforementioned properties of relativity and reversibility.
Furthermore, to better utilize the pre-trained weights of 2D models, we also seek to ensure that when the relative position on the \(z\)-axis is 0, the results are equal to the 2D counterparts.
Specifically, the transformation matrix at $(t_x, t_y, t_z)$ is given as:

\begin{equation}
    \mathcal{R}(t_x, t_y, t_z) =     
    \begin{bmatrix} 
        \mathcal{R}_{x, z}(t_x, t_z) & \\
         & \mathcal{R}_{y, z}(t_y, t_z) 
    \end{bmatrix},
\end{equation}
where
\begin{equation}
    \begin{aligned}
        & \forall a \in \{x, y\}, \mathcal{R}_{a, z}(t_{a}, t_z) = \\
        &\mathrm{diag}\left(
            T(\omega_{a, 0} t_a + \omega_{z, 0} t_z),
            \dots,
            T(\omega_{a, \frac{d}{2}-1}t_a + \omega_{z, \frac{d}{2}-1}t_z)
        \right).
    \end{aligned}
\end{equation}
In our experiments, we set $\omega_{x, i} = b_x^{-\frac{2i}{d}}, \omega_{y, i} = b_y^{-\frac{2i}{d}}, \omega_{z, i} = b_z^{-\frac{2i + 1}{d}}$, where $b_x = b_y = 10000, b_z = 2333$.
We show with our specifically chosen $\omega_{x, i}, \omega_{y, i}, \omega_{z, i}$, such approach maintains the properties of relativity and reversibility.

\paragraph{Relativity}
Obviously, we only need to show \(\mathcal{R}_{x, z}(\cdot, \cdot)\) satisfies relativity. For each diagonal entry, we have:
\begin{align}
     &R^\mathsf{T}(\omega_{x, i} t_x^{(1)} + \omega_{z, i} t_z^{(1)})
        R(\omega_{x, i} t_x^{(2)} + \omega_{z, i} t_z^{(2)}) \\
    = &R(\omega_{x, i} (t_x^{(2)} - t_x^{(1)}) + \omega_{z, i} (t_z^{(2)} -t_z^{(1)})).
\end{align}
Therefore, \(\mathcal{R}^{\mathsf{T}}_{x, z}(t_x^{(1)}, t_z^{(1)})\mathcal{R}_{x, z}(t_x^{(2)}, t_z^{(2)}) = \mathcal{R}_{x, z}(t_x^{(2)} - t_x^{(1)}, t_z^{(2)} - t_z^{(1)})\).

\paragraph{Reversibility}

Similarly, we only need to show \(\mathcal{R}_{x, z}(\cdot, \cdot)\) satisfies reversibility.
Note that the values we specifically choose for $\omega_{x, i}, \omega_{y, i}, \omega_{z, i}$ are algebraic numbers, the following theorem helps prove the result.

\begin{theorem}
    If both $\omega_x$ and $\omega_z$ are non-zero algebraic numbers, and $\frac{\omega_x}{\omega_z}$ is an irrational number, then for any $(t_x^{(1)}, t_z^{(1)}), (t_x^{(2)}, t_z^{(2)}) \in \mathbb{Z}^2$ where $(t_x^{(1)}, t_z^{(1)}) \neq (t_x^{(2)}, t_z^{(2)})$,
    it holds that \(\mathcal{R}_{x, z}(t_x^{(1)}, t_z^{(1)}) \ne \mathcal{R}_{x, z}(t_x^{(2)}, t_z^{(2)})\).
\end{theorem}
\begin{proof}
    The proposition is equivalent to:
    \begin{equation}
        \omega_{x} t_x^{(1)} + \omega_{z} t_z^{(1)} \not\equiv \omega_{x} t_x^{(2)} + \omega_{z} t_z^{(2)} \pmod{2\pi}.
    \end{equation}

    Proof by contradiction. Assume there exists $(t_x^{(1)}, t_z^{(1)}), (t_x^{(2)}, t_z^{(2)}) \in \mathbb{Z}^2$ and $(t_x^{(1)}, t_z^{(1)}) \neq (t_x^{(2)}, t_z^{(2)})$, \(k \in \mathbb{Z}\) such that
    \begin{equation}\label{eq:RoPE-Reversibility-proof}
        2k\pi = \omega_{x} (t_x^{(2)} - t_x^{(1)}) + \omega_{z} (t_z^{(2)} - t_z^{(1)}).
    \end{equation}

    If $k = 0$, then $t_x^{(2)} \neq t_x^{(1)}$ (otherwise, from $0 = \omega_{x} (t_x^{(2)} - t_x^{(1)}) + \omega_{z} (t_z^{(2)} - t_z^{(1)})$, we would also have $t_z^{(2)} = t_z^{(1)}$, which contradicts $(t_x^{(1)}, t_z^{(1)}) \neq (t_x^{(2)}, t_z^{(2)})$).
    Therefore, from Equation~\ref{eq:RoPE-Reversibility-proof} we have:
    \begin{equation}
        \frac{\omega_x}{\omega_z} = -\frac{t_z^{(2)} - t_z^{(1)}}{t_x^{(2)} - t_x^{(1)}} \in \mathbb{Q},
    \end{equation}
    which contradicts the condition that $\frac{\omega_x}{\omega_z}$ is irrational.
    
    If $k \neq 0$, since the product and sum of algebraic numbers are still algebraic, from Equation~\ref{eq:RoPE-Reversibility-proof}, we have \\
    $\pi = \dfrac{\omega_{x} (t_x^{(2)} - t_x^{(1)}) + \omega_{z} (t_z^{(2)} - t_z^{(1)})}{2k}$ is also algebraic. However, this contradicts the well-known fact that $\pi$ is transcendental.
\end{proof}

\subsection{Experiment}

To validate the effectiveness of the proposed Additive 3D RoPE, we computed the rotation angles corresponding to all positions on images of a typical size. 
Note that the $x$ and $y$ dimensions are symmetric in Additive 3D RoPE, therefore, in our experiments, we set the length of the $y$ dimension to 1, and only focus on the $x$ and $z$ dimensions.
Accordingly, for each $0 \leq i \leq d/4 - 1$ where $d$ is the dimension of the query or key vector $\bm{a}$), we examined the distribution of these angles in terms of $\bmod 2\pi$ representation.

For image size to be validated, we set it as $8 \times 16 \times 1 (z \times x \times y)$.
Notice that that when the z-coordinate $t_z = 0$, the rotation angle becomes $\omega_{x, i} t_x + \omega_{z, i} t_z = \omega_{x, i} t_x$, which is equivalent to 2D RoPE.
Therefore, we also choose 2D image size of $1 \times 16 \times 1 (z \times x \times y)$ for comparison with 2D RoPE.

For each $0 \leq i \leq d/4 -1$, we consider the results of rotating angles of all positions modulo $2\pi$ to obtain $\Theta_i$. 
Subsequently, we calculate the differences between all adjacent elements after sorting $\Theta_i$, resulting in $\Delta_i$. 
Finally, we calculate the minimum value among these differences as $\min \Delta_i$.
The formal definitions of \(\Theta_i, \Delta_i\) are as follows:
\begin{alignat*}{3}
    \Theta_i &= \{ (\omega_{x, i} t_x + \omega_{z, i} t_z ) \bmod 2\pi \mid 0 \le t_x < x, 0 \le t_z < z \} \\
        &= \{ \theta_1, \theta_2, \cdots, \theta_{xz} \} \quad (0 \le \theta_1 \le \theta_2 \le \cdots \le \theta_{xz} < 2\pi) \\
    \Delta_i &= \{ \theta_{i+1} - \theta_{i} \mid 1 \le i \le xz-1 \} \cup \{ \theta_0 + 2\pi - \theta_{xz}\} 
\end{alignat*}

In addition, we also attempt to use the Euclidean distance between the positions after rotating $(0, 1) \in \mathbb{R}^2$ by angles of $\alpha$ and $\beta$ as an indicator to quantify the "difference" between these two rotation angles, as follows:
\begin{align*}
    D_i = \{ & \sqrt{(\cos \theta_{i+1} - \cos \theta_{i})^2 + (\sin \theta_{i+1} - \sin \theta_{i})^2} \\
            &\mid 1 \le i \le xz\},
\end{align*}
where \(\theta_{xz+1}\) is defined as \(\theta_1\).

Since there is a considerable number of $\theta_i$ values, the distribution of them on the unit circle $x^2 + y^2 = 1$ is relatively dense.
In this scenario, $\sqrt{(\cos \theta_{i+1} - \cos \theta_{i})^2 + (\sin \theta_{i+1} - \sin \theta_{i})^2} \approx (\theta_{i+1} - \theta_{i}) \bmod 2\pi$, implying that the arc length corresponding to this central angle will be approximately equal to the chord length.
Consequently, the results obtained from these two evaluation metrics will be quite close.
For simplicity, only the differences of adjacent angles are presented below.

\begin{figure}[htbp]
    \centering
    \includegraphics[width=\linewidth]{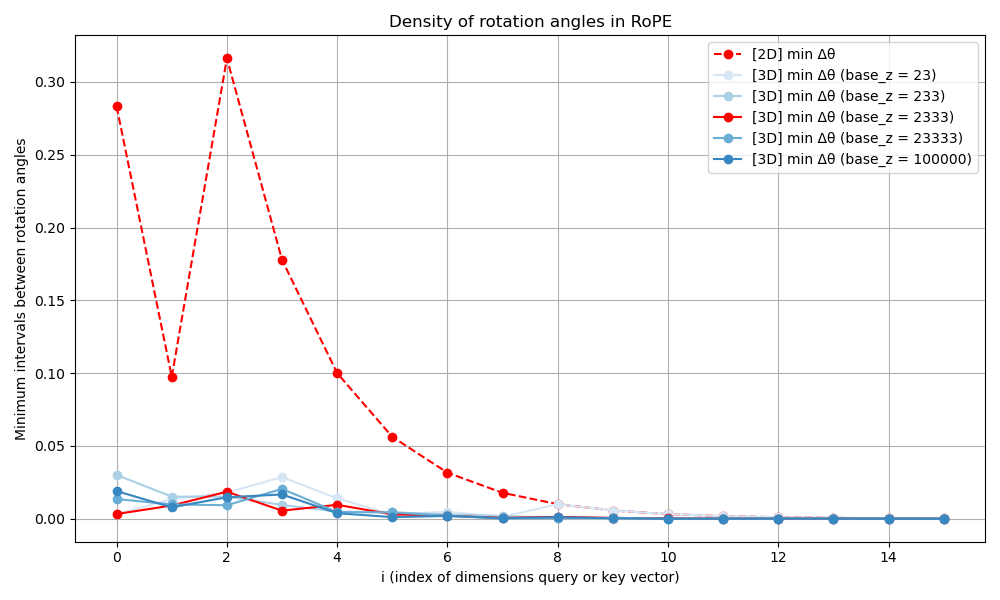}
    \caption{
        Minimum intervals between rotation angles for different $i$.
        $i$ is the index of dimensions $d$ of query or key vector $\bm{a}$.
    }
    \label{fig:rope-exp}
\end{figure}

It can be observed that when using Additive 3D RoPE on 3D images, the minimum intervals between rotation angles are approximately in the range of $10^{-3}$ to $10^{-5}$, which preserves good numerical precision with 32bit floating point numbers.
As for the comparison with 2D RoPE, the average result of $\frac{\min \Delta_i^{(2D)}}{\min \Delta_i^{(3D)}}$ for all $i$ is 26.30, which benefits 3D medical images processing.
By setting $b_z < b_x$, which implies $\omega_{z, i} > \omega_{x, i}$, Additive 3D RoPE gains anisotropy, allowing it to carry cross-plane positional information.
The effect of this choice on the distribution of rotation angles is that rotation angles for positions with the same $z$ coordinate tend to cluster together, while positions with different $z$ coordinates in the entire image give rise to distinct clusters of rotation angles.
Thus, the difference between our results and the ideal ratio of 16.00 for both cases with uniform distribution reflects the anisotropic information.

As shown in Figure~\ref{fig:rope-exp}, we can also observe the relationship between $b_z$ and $\min \Delta_i$: 
if $b_z$ is large, indicating that $\omega_{z, i}$ is small, all rotation angles are densely distributed, resulting in a smaller $\min \Delta_i$; 
whereas if $b_z$ is small, although $\min \Delta_i$ is better, larger $\omega_{z, i}$ can lead to rotation angles for positions that are far apart become too close in terms $\bmod 2\pi$, which contradicts the relativity pursued by RoPE. 
In conclusion, we finally set $b_z = 2333$ in our work.

\begin{table*}[htbp]
    \centering
    \fontsize{9}{11}\selectfont
    \begin{tabular}{l|cccc}
    \toprule
    \textbf{Configuration} & UniMiSS & SMIT & EVA-02-B & Ours \\
    \midrule
    initial learning rate & \multicolumn{4}{c}{5e-8 / 1e-7 / 5e-7 / 1e-6 / 5e-6 / 1e-5 / 5e-5 / 1e-4 / 5e-4} \\
    lr decay factor & \multicolumn{4}{c}{0.1} \\
    training epochs & \multicolumn{4}{c}{100} \\
    lr schedule & \multicolumn{4}{c}{lr step at epoch 50, 75} \\
    batch size & \multicolumn{4}{c}{32} \\
    optimizer & \multicolumn{4}{c}{AdamW \cite{loshchilov2018decoupled}} \\
    weight decay & \multicolumn{4}{c}{1e-2} \\
    optimizer momentum & \multicolumn{4}{c}{$\beta_{1}=0.9, \beta_{2}=0.999$} \\
    fine-tuning input size & $96 \times 96 \times 96$ & $96 \times 96 \times 96$ & $28 \times 224 \times 224$ & $160 \times 160 \times 160$ \\
    \midrule
    pre-training input size & $16 \times 96 \times 96$ & $96 \times 96 \times 96$ & $224 \times 224$ & $80 \times 160 \times 160$ \\
    \bottomrule
    \end{tabular}
    \caption{Classification fine-tuning settings}
    \label{table:cfg}
\end{table*}

\section{Interpretation of PDR}

In fact, there exists a constructive interpretation when using PDR with the soft token representation.
Suppose the encoder outputs a $D \times H \times W$ grid of token distributions, let \(G\) denote the set of discrete coordinates of all cells within the grid.
For each position \(s \in G\) on the grid, we define a random variable $I_s$ denoting the index of the token in the codebook at \(s\) and its distribution to be \(q^{(s)}\).
We define another random variable $S$ representing a random position on the grid that is uniformly distributed over $G$.
It can be shown that the averaged distribution of all token distributions on the grid is exactly the distribution of $I_S$:
\begin{align}
    \P(I_S=i) 
    &= \sum_{s \in G_{D, H, W}} \P(I_S=i \mid S=s)\P(S=s) \\
    &= \frac{1}{DHW}\sum_{s \in G_{D, H, W}} \P(I_s=i).
\end{align}
Intuitively, when the distribution of $I_S$ is close to the uniform distribution, individual tokens occur at a randomly sampled position with close probabilities, which benefits the diversity of the learned distributions.

\section{Pre-training Data and Pre-processing}

From each dataset, we exclude any corrupted data or those lacking consistent spacing information.
During pre-processing, for each image, we determine its depth dimension, i.e., the dimension along which the image is scanned, and move the depth dimension to the first spatial dimension.
The depth dimension is determined based on the following priorities: 1) adhering to official data description if available; 2) extracting from image metadata if available; 3) using the dimension with the most anisotropic spacing compared to other dimensions.
We crop the image to the region with intensities above zero for RGB and gray scale images, and above the 0.5 percentile for others.
If the shorter in-plane size is greater than 512, we resize the image to have a shorter in-plane size of 512 with area interpolation.
The full list of datasets and the amount of data in each dataset used for pre-training are presented in Table~\ref{table:datasets}.
The number of slices for 3D images is counted along the depth dimension.
In total, 48,344 3D images and 348,991 2D images, amounting to 9,096,684 slices, are used for pre-training.

\section{More Reconstruction Results}

We show more reconstruction results by SPAD-VT in Figure~\ref{fig:tokeninzer-eval-more}.
SPAD-VT can reconstruct various anatomical structure with diverse imaging modalities and spatial properties.

\section{Classification Experiments}


Our experimental settings on MedMNIST v2 largely follow the official baselines \cite{medmnistv2}, including the regularization for the Adrenal and Vessel datasets. 
For fair comparisons, we search for optimal learning rate when fine-tuning each model on each dataset.
To utilize the pre-trained weights of models, we adopt the following adjustments.
First, for each model, each input image is interpolated into the size that is compatible with the input size used by the model during pre-training.
Second, the 3D images of MedMNIST v2 have only one channel. For each model that is pre-trained on images with three channels (such as RGB images), we duplicate each input image along the channel dimension into three channels.
The detailed recipe is provided in Table~\ref{table:cfg}.


\section{Segmentation Experiments}

\subsection{Experimental Settings}

For experiments on BTCV, a $48 \times 192 \times 192$ patch from each input image during training for all models but Swin UNETR.
The model structure of Swin UNETR only supports input size that is divisible by 32, and we thus follow the official implementation of Swin UNETR\footnote{\url{https://github.com/Project-MONAI/research-contributions/tree/main/SwinUNETR/BTCV}} to use a crop size of $96 \times 96 \times 96$ specifically for it.
Following Swin UNETR, we use a soft tissue window of $[-175, 250]$ for clipping the CT intensities, then rescale intensities to [-1, 1].
For experiments on CHAOS, the crop size during training is $32 \times 192 \times 192$ for all models.
For each MRI image, we normalize it by subtracting the mean of intensities and dividing standard deviation of intensities.
For both tasks, all models are trained for 50K steps and segmentation results are obtained by sliding window inference with a overlap ratio of 0.75.
The remained hyper-parameters settings for each model follow the corresponding official implementation.

\subsection{Evaluation Metrics}

Given a prediction set of points \(\tilde{V}\) and a reference set \(V\), we evaluate the segmentation results using the Dice coefficient, average symmetric surface distance (ASSD), and Hausdorff distance (HD) as follows:
\small
\begin{gather}
    \mathrm{Dice}(\tilde{V}, V) =
    \frac{2 \lvert \tilde{V} \cap V \rvert}{\lvert V \rvert + \lvert \tilde{V} \rvert }, \\
    \begin{aligned}
        \mathrm{ASSD}(\tilde{V}, V) &= \\
        \frac{1}{\lvert \partial V \rvert + \lvert \partial \tilde{V} \rvert}
        &\left(
            \sum_{s \in \partial V} d(s,\partial \tilde{V}) + 
            \sum_{s \in \partial \tilde{V}} d(s, \partial V)
        \right),
    \end{aligned} \\
    \mathrm{HD}(\tilde{V}, V) =
    \max \left \{
        \max_{s \in \partial V} d(s,\partial \tilde{V}),
        \max_{s \in \partial \tilde{V}} d(s,\partial V)
    \right \},
\end{gather}
\normalsize
where \(\partial V\) is the boundary of \(V\), and
\begin{equation}
    d(s, V) = \min_{t \in V} d(s, t),
\end{equation}
where \(d(s, t)\) is the Euclidean distance between points \(s\) and \(t\).
More specific details of the evaluation metrics can be found in the official documentation of the CHAOS challenge\footnote{\url{https://github.com/emrekavur/CHAOS-evaluation/blob/master/CHAOS_Metrics_and_Evaluation_Method.pdf}}.
For results on BTCV validation set, we use the MONAI's implementation for the metrics to evaluate them.
For results on CHAOS test set, we submit them to the challenge\footnote{\url{https://chaos.grand-challenge.org}} and report the evaluation results returned by the website.

\newpage

\begin{figure*}[htbp]
    \centering
    \includegraphics[width=\linewidth]{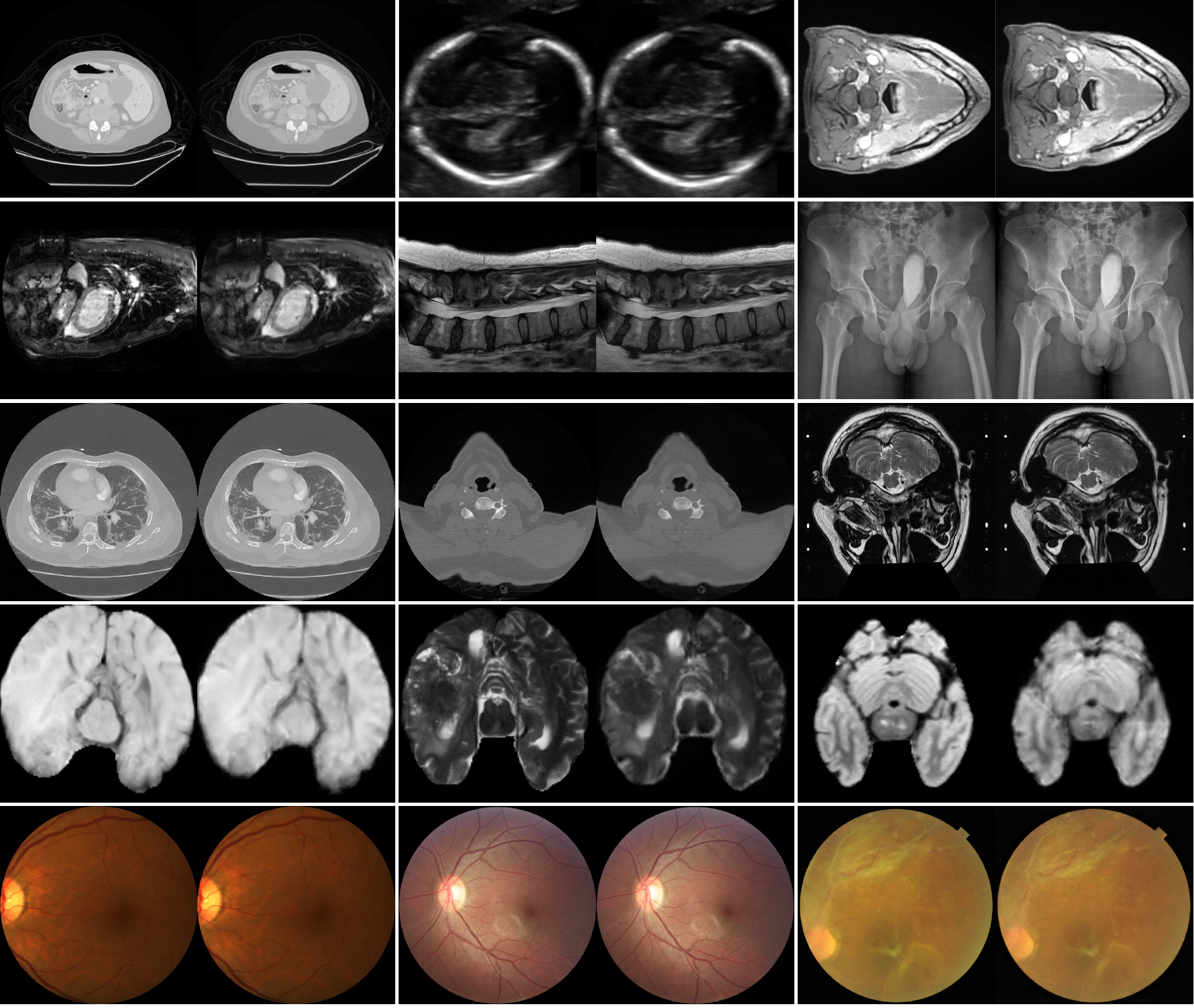}
    \caption{
        More visual tokenizer reconstruction results on unseen images.
        The first row:
            (left) abdomen CT image MSD-Pancreas;
            (middle) fetal ultrasound image from HC18;
            (right) head-and-neck MRI image from HaN-Seg.
        The second row:
            (left): heart MRI image from MSD-Heart;
            (middle): abdomen CT image from FLARE 2022;
            (right): pelvix X-ray image from PelviXNet.
        The third row:
            (left): lung image from STOIC;
            (middle): spine CT image from RSNA-2022-CSFD;
            (right): brain MRI image from crossMoDA2022.
        The forth row:
            (from left to right) T1-weighted, T2-weighted, and T2-weighted FLAIR brain MRI images from BraTS-MEN.
        The fifth row:
            (from left to right) fundus photograph from iChallenge-ADAM, iChallenge-REFUGE2, and RFMiD, respectively.
    }
    \label{fig:tokeninzer-eval-more}
\end{figure*}

\newpage

\onecolumn

\begin{longtable}{lp{6.5cm}lllll}
    \caption{Dataset used for pre-training}
    \label{table:datasets}
    \\
    \toprule
    \textbf{No.} &
      \textbf{Dataset} &
      \textbf{Dim.} &
      \textbf{Modality} &
      \textbf{Location} &
      \textbf{\#images} &
      \textbf{\#slices} \\
    \midrule
    \endfirsthead
    \toprule
    \textbf{No.} &
      \textbf{Dataset} &
      \textbf{Dim.} &
      \textbf{Modality} &
      \textbf{Location} &
      \textbf{\#images} &
      \textbf{\#slices} \\
    \midrule
    \endhead
    \bottomrule
    \endlastfoot
    1  & ACDC \cite{acdc-dataset}                                      & 3D & cine MRI & heart       & 150     & 1,489     \\
    \midrule
    2  & AMOS22 \cite{amos22_nips2022}                                 & 3D & CT; MRI  & abdomen     & 600     & 85,745    \\
    \midrule
    3  & BCV/Abdomen \cite{btcv_miccai2015}                            & 3D & CT       & abdomen     & 50      & 6,166     \\
    \midrule
    4  & BCV/Cervix \cite{btcv_miccai2015}                             & 3D & CT       & cervix      & 50      & 9,078     \\
    \midrule
    5  & BrainPTM-2021 \cite{avital2019neural,nelkenbaum2020automatic} & 3D & MRI      & brain       & 150     & 20,930    \\
    \midrule
    6 &
      BraTS2023/GLI \cite{Karargyris2023,baid2021rsnaasnrmiccai,brats_tmi2015,Bakas2017,TCGA-GBM,TCGA-LGG} &
      3D &
      MRI &
      brain &
      5,880 &
      817,632 \\
      \midrule
    7  & BraTS2023/MEN \cite{Karargyris2023,labella2023asnrmiccai}    & 3D & MRI      & brain       & 4,564   & 630,137   \\
    \midrule
    8  & BraTS2023/MET \cite{Karargyris2023,moawad2023brain}          & 3D & MRI      & brain       & 660     & 90,626    \\
    \midrule
    9  & BraTS2023/PED \cite{Karargyris2023,kazerooni2023brain}       & 3D & MRI      & brain       & 576     & 79,654    \\
    \midrule
    10 & BraTS2023/SSA \cite{Karargyris2023,adewole2023brain}         & 3D & MRI      & brain       & 300     & 41,585    \\
    \midrule
    11 & BUSI \cite{BUSI}                                              & 2D & US       & breast      & 780     & 780       \\
    \midrule
    12 &
      CGMH Pelvis\footnote{\url{https://www.kaggle.com/datasets/tommyngx/cgmh-pelvisseg}} &
      2D &
      XR &
      pelvis &
      400 &
      400 \\
      \midrule
    13 & CHAOS \cite{CHAOS2021,CHAOSdata2019,kavur2019}              & 3D & MRI      & brain       & 120     & 8,975     \\
    \midrule
    14 & CHASE\_DB1 \cite{CHASE_DB1}                                   & 2D & RGB      & fundus      & 28      & 28        \\
    \midrule
    15 & CheXpert \cite{chexpert}                                      & 2D & XR       & chest       & 224,316 & 224,316   \\
    \midrule
    16 & CHUAC \cite{CARBALLAL20181}                                   & 2D & RGB      & fundus      & 30      & 30        \\
    \midrule
    17 & Chákṣu \cite{Kumar2023}                                       & 2D & RGB      & fundus      & 1,345   & 1,345     \\
    \midrule
    18 & crossMoDA2022 \cite{Dorent_2023}                              & 3D & MRI      & brain       & 484     & 34,870    \\
    \midrule
    19 & EPISURG \cite{EPISURG}                                        & 3D & MRI      & brain       & 698     & 124,804   \\
    \midrule
    20 & FLARE22\footnote{\url{https://flare22.grand-challenge.org}} & 3D & CT       & abdomen     & 2,100   & 371,660   \\
    \midrule
    21 & HaN-Seg \cite{HaN-Seg}                                        & 3D & CT; MRI  & head-neck   & 84      & 11,177    \\
    \midrule
    22 & HC18 \cite{HC18}                                              & 2D & US       & fetal       & 1,334   & 1,334     \\
    \midrule
    23 & HNSCC \cite{HC18-dataset}                                     & 3D & CT; MRI  & head-neck   & 1,246   & 266,625   \\
    \midrule
    24 & iChallenge/ADAM\footnote{\url{https://ichallenges.grand-challenge.org}}                                               & 2D & RGB      & fundus      & 1,200   & 1,200     \\
    \midrule
    25 & iChallenge/GAMMA                                              & 2D & RGB      & fundus      & 300     & 300       \\
    \midrule
    26 & iChallenge/PALM                                               & 2D & RGB      & fundus      & 1,200   & 1,200     \\
    \midrule
    27 & iChallenge/REFUGE2                                            & 2D & RGB      & fundus      & 2,000   & 2,000     \\
    \midrule
    28 & IDRiD \cite{IDRiD,IDRiD-dataset}                             & 2D & RGB      & fundus      & 588     & 588       \\
    \midrule
    29 & IXI\footnote{\url{https://brain-development.org/ixi-dataset}} & 3D & MRI      & brain       & 2,706   & 317,711   \\
    \midrule
    30 & RFMiD \cite{RFMiD}                                            & 2D & RGB      & fundus      & 3,200   & 3,200     \\
    \midrule
    31 & LIDC-IDRI \cite{LIDC-IDRI,LIDC-IDRI-TCIA,TCIA}              & 3D & CT       & chest       & 1,018   & 243,958   \\
    \midrule
    32 & MRSpineSeg \cite{MRSpineSeg1,MRSpineSeg2}                    & 3D & MRI      & spine       & 387     & 4,881     \\
    \midrule
    33 & MSD/Task02\_Heart \cite{MSD-data,MSD}                        & 3D & MRI      & heart       & 30      & 3,568     \\
    \midrule
    34 & MSD/Task03\_Liver \cite{MSD-data,MSD}                        & 3D & MRI      & abdomen     & 201     & 85,246    \\
    \midrule
    35 & MSD/Task04\_Hippocampus \cite{MSD-data,MSD}                  & 3D & MRI      & hippocampus & 390     & 13,783    \\
    \midrule
    36 & MSD/Task05\_Prostate \cite{MSD-data,MSD}                     & 3D & MRI      & prostate    & 96      & 1,752     \\
    \midrule
    37 & MSD/Task06\_Lung \cite{MSD-data,MSD}                         & 3D & CT       & chest       & 95      & 26,545    \\
    \midrule
    38 & MSD/Task07\_Pancreas \cite{MSD-data,MSD}                     & 3D & CT       & abdomen     & 420     & 40,263    \\
    \midrule
    39 & MSD/Task08\_HepaticVessel \cite{MSD-data,MSD}                & 3D & CT       & abdomen     & 443     & 31,668    \\
    \midrule
    40 & MSD/Task09\_Spleen \cite{MSD-data,MSD}                       & 3D & CT       & abdomen     & 61      & 5,277     \\
    \midrule
    41 & MSD/Task10\_Colon \cite{MSD-data,MSD}                        & 3D & CT       & abdomen     & 190     & 20,102    \\
    \midrule
    42 & NCTCT \cite{Johnson2008,NCTCT-TCIA,TCIA}                    & 3D & CT       & abdomen     & 1,735   & 938,496   \\
    \midrule
    43 & NIH Chest X-ray \cite{ChestX-Ray8}                            & 2D & XR       & chest       & 112,120 & 112,120   \\
    \midrule
    44 & PelviXNet \cite{PelviXNet}                                    & 2D & XR       & pelvis      & 150     & 150       \\
    \midrule
    45 & PI-CAI \cite{PI-CAI}                                          & 3D & MRI      & prostate    & 7,488   & 162,891   \\
    \midrule
    46 & PROSTATE-DIAGNOSIS \cite{PROSTATE-DIAGNOSIS,TCIA}            & 3D & MRI      & prostate    & 261     & 8,450     \\
    \midrule
    47 & PROSTATE-MRI \cite{prostate-mri}                              & 3D & MRI      & prostate    & 182     & 22,036    \\
    \midrule
    48 & Prostate158 \cite{Prostate158}                                & 3D & MRI      & prostate    & 417     & 10,287    \\
    \midrule
    49 & RibFrac \cite{RibFrac}                                        & 3D & CT       & rib         & 500     & 180,547   \\
    \midrule
    50 & RSNA-2020-PED \cite{RSNA-2020-PED}                            & 3D & CT       & chest       & 7,929   & 1,937,863 \\
    \midrule
    51 & RSNA-2022-CSFD \cite{RSNA-2022-CSFD}                          & 3D & CT       & spine       & 2,022   & 712,919   \\
    \midrule
    52 & STOIC \cite{STOIC}                                            & 3D & CT       & chest       & 2,000   & 867,376   \\
    \midrule
    53 & TotalSegmentator \cite{TotalSegmentator}                      & 3D & CT       & body        & 1,203   & 312,395   \\
    \midrule
    54 & VerSe \cite{VerSe1,VerSe2,VerSe3}                           & 3D & CT       & vertebrae   & 374     & 151,396   \\
    \midrule
    55 & V-S-SEG \cite{V-S-SEG,V-S-SEG-TCIA,TCIA}                    & 3D & MRI      & brain       & 484     & 47,130    \\
\end{longtable}

\twocolumn

\end{document}